\documentclass[lettersize,journal]{IEEEtran}

\usepackage[T1]{fontenc}
\usepackage[utf8]{inputenc}

\usepackage{amsmath,amsfonts}
\usepackage{amssymb} % math symbols
\usepackage{amsthm}  % theorem/proof environments (IEEEtran-compatible)
\usepackage{algorithm}
\usepackage{algpseudocode}
\usepackage{array}
\usepackage[caption=false,font=normalsize,labelfont=sf,textfont=sf]{subfig}
\usepackage{textcomp}
\usepackage{stfloats}
\usepackage{url}
\usepackage{verbatim}
\usepackage{graphicx}
\def\BibTeX{{\rm B\kern-.05em{\sc i\kern-.025em b}\kern-.08em
    T\kern-.1667em\lower.7ex\hbox{E}\kern-.125emX}}
\usepackage{balance}

\newtheorem{theorem}{Theorem}

\newtheorem{lemma}{Lemma}

\newtheorem{condition}{Condition}

\theoremstyle{remark}
\newtheorem{remark}{Remark}

\begin{document}
\title{Early-Terminable Energy-Safe Iterative Coupling for Parallel Simulation of Partitioned Port-Hamiltonian Systems}
\author{Qi Wei, Jianfeng Tao, Hongyu Nie, Wangtao Tan
\thanks{This work was supported by the National
Natural Science Foundation of China (Grant No. 52075320). 

The authors are with the School of Mechanical Engineering, Shanghai Jiao Tong University, Shanghai 200240, China. (e-mail: jftao@sjtu.edu.cn)
}}
 
\markboth{}%
{Early-Terminable Energy-Safe Iterative Coupling for Parallel Simulation of Partitioned Port-Hamiltonian Systems}

\maketitle

\begin{abstract}
Parallel simulation of robotic systems requires partitioning the dynamics into coupled subsystems. Finite-iteration coupling across the partition boundary can inject spurious energy, even when each subsystem is passive. We propose an early-terminable, energy-safe coupling interface for port-Hamiltonian subsystems based on Douglas--Rachford splitting in wave (scattering) coordinates. The wave-domain formulation reduces passivity to norm inequalities and coupling to orthogonality. Within this setting, the deep correspondence between monotone operator theory and discrete passivity can be exploited to construct a Douglas--Rachford inner iteration whose Fejér monotonicity provides algorithmic dissipation. 
Under passivity of the subsystem integrators and an impedance-tuning condition, the proposed method guarantees discrete passivity of the augmented storage for any finite inner-iteration budget and converges to the monolithic discretization as the budget increases. Experiments on a linear--Duffing coupled-oscillator benchmark support the finite-iteration energy inequality at numerical roundoff ($10^{-14}$ in double precision), with state-error metrics decreasing over the tested inner-iteration budgets.

\end{abstract}


\section{Introduction}
\IEEEPARstart{P}{artitioned} time stepping is a practical necessity when simulating stiff, multi-physics robotic systems under real-time constraints \cite{negrut2014parallel}. A single coupled solve across all degrees of freedom can exceed the available time budget, making monolithic time stepping impractical for online applications.
Port-based model partitioning splits a system at its physical interfaces into coupled subsystems and advances each in parallel, offering a path to substantial wall-clock reduction. 
The subsystems remain coupled through the interface, however, and discrete exchange across that interface breaks the continuous-time power balance \cite{benedikt2013nepce,sadjina2017energy}, yielding spurious energy injection (effective negative damping) even when every subsystem is individually passive \cite{sadjina2024energy,gonzalez2019energy,rodriguez2022energy}. 
In time-critical simulation the inner iteration cannot run to convergence. It must terminate after a prescribed budget, and the interface signal delivered at termination must already be energetically admissible.

Existing partitioned coupling strategies fall into three families. The first, explicit coupling (zero-order-hold or extrapolation), avoids iteration but can destabilize under stiff coupling \cite{schweizer2015explicit,gomes2018co,blochwitz2011functional,chen2021explicit}. 
The second, macro-step iterations (Gauss--Seidel, Jacobi, relaxation), improve accuracy at the cost of repeated evaluations, yet their stability guarantees typically concern the converged interface solution \cite{peiret2018multibody,raoofian2024nonsmooth,dai2023model,schweizer2016implicit}. 
The third, Transmission Line Modeling (TLM), preserves passivity through physically motivated time delays \cite{krus2011robust,braun2022numerically,braun2013multi}, at the price of high-frequency distortion and parameter tuning \cite{braun2024transmission,barbierato2024locally,burton1993analysis}. 
What remains missing across these strategies is a guarantee that the coupled macro-step remains energy-safe when the inner iteration is cut short.

A second group of methods supplies the theoretical building blocks for such an interface. Wave (scattering) variables convert port passivity to a norm condition and lossless coupling to an orthogonal constraint \cite{cervera2007interconnection}, yet in co-simulation their use has mostly been associated with TLM-style explicit delays. Port-Hamiltonian formulations provide structure-preserving integrators that respect the continuous energy balance in monolithic simulation \cite{celledoni2017energy,argus2021theory,ehrhardt2025structure,parks2017variational}; retaining these properties under partitioned coupling has proven difficult \cite{bartel2023operator,bartel2025splitting}. These two tools have been developed in separate communities. A structural property of the wave domain makes their unification possible. Passivity and lossless coupling both reduce to norm conditions in wave coordinates, making the interface problem amenable to monotone operator methods. This letter develops the resulting synthesis.

In this letter, we propose an early-terminable energy-safe coupling interface based on wave-domain DR splitting. At each macro-step, a Douglas--Rachford (DR) inner iteration reconciles a linear lossless wave coupling constraint with the discrete-time port behavior of the subsystem integrators. The inner-iteration budget $K_n$ provides a tunable accuracy--effort trade-off, with an augmented-storage energy guarantee under any finite $K_n$. The main contributions are as follows. 

1) \textbf{An early-terminable energy-safe interface for parallel coupling:}
We introduce a parallelizable iterative coupling interface in wave (scattering) variables that enforces a linear lossless interconnection via DR inner iterations. The method accommodates nonlinear subsystem dynamics under a linear lossless interconnection (Algorithm~\ref{alg:scattering_iterative_coupling}). An open-source C++ implementation is available at \url{https://github.com/V7CN/wave-domain-iterative-coupling}.

2) \textbf{Finite-iteration passivity guarantee:}
Under the standing conditions of Section~IV, we establish a discrete passivity guarantee for the macro-step via an augmented storage function, for any finite inner-iteration budget (Theorem~\ref{thm:augmented_storage}).

3) \textbf{Convergence to the monolithic discretization:}
As the inner-iteration budget increases, the partitioned update provably converges to the monolithic discrete-time update induced by the same integrator (Theorem~\ref{thm:kn_to_infty_monolithic}).

The remainder of the letter is organized as follows. Section~II models the subsystems and coupling in the wave domain and formulates the interface as a fixed-point problem. Section~III presents the iterative coupling algorithm. Section~IV establishes passivity and convergence properties. Section~V reports numerical experiments. Section~VI concludes with open questions and extensions.

\section{Coupling Problem in Wave Domain: From Power Ports to Fixed-Point Form}
\label{sec:problem_formulation}

\subsection{Port-Hamiltonian Subsystems}
\label{subsec:ph_passivity}

Consider a robotic system partitioned into two subsystems $A$ and $B$. Each subsystem $i \in \{A, B\}$ has state $x_i \in \mathbb{R}^{n_i}$ and Hamiltonian $H_i: \mathbb{R}^{n_i} \to \mathbb{R}$ (the stored energy function). The subsystem exchanges energy with its environment through a power port with port variables effort $e_i \in \mathbb{R}^m$ and flow $f_i \in \mathbb{R}^m$, whose inner product $e_i^\top f_i$ is the instantaneous power injected into subsystem $i$ through the port.

Each subsystem is described in standard port-Hamiltonian (pH) form. The dynamics of subsystem $i$ are
\begin{equation}
\begin{aligned}
\dot{x}_i &= \big(J_i(x_i) - R_i(x_i)\big) \nabla H_i(x_i) + G_i(x_i) e_i, \\[2pt]
f_i &= G_i(x_i)^\top \nabla H_i(x_i),
\end{aligned}
\label{eq:ph_subsystem}
\end{equation}
where $J_i(x_i)^\top = -J_i(x_i)$ encodes the internal interconnection structure (energy routing), $R_i(x_i)^\top = R_i(x_i) \succeq 0$ encodes dissipation, and $G_i(x_i)$ maps the port effort into the state space. From the skew-symmetry of $J_i$ and the positive semi-definiteness of $R_i$, we obtain the continuous-time passivity inequality
\begin{equation}
\dot{H}_i(x_i) \leq e_i^\top f_i,
\label{eq:continuous_passivity} 
\end{equation}
which states that the rate of change of stored energy in subsystem $i$ does not exceed the power injected through the port. Inequality~\eqref{eq:continuous_passivity} is the physical foundation of all subsequent analysis in this paper, and the discrete numerical computation must preserve this property.

We consider a lossless interconnection. For two directly coupled ports, the coupling constraints are $e_A = e_B$ and $f_A + f_B = 0$, which imply power conservation $e_A^\top f_A + e_B^\top f_B = 0$. Violating these constraints can inject spurious energy.

\begin{figure}[!t]
\centering
\includegraphics[width=\linewidth]{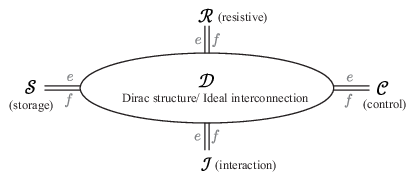}%
\caption{Port-Hamiltonian subsystem with power port $(e_i, f_i)$.}
\label{fig:ph_subsystem}
\end{figure}

\subsection{Scattering Transformation to the Wave Domain}
\label{subsec:scattering}

Scattering variables have two key properties. \textbf{(1) Port power becomes a difference of Euclidean norms}, so passivity reduces to a norm inequality. \textbf{(2) Lossless coupling becomes an orthogonal linear constraint}, so the interconnection is energy-neutral by construction.

Choose a positive parameter $\gamma > 0$ (the scattering impedance; see Section~\ref{subsec:standing_conditions}), and define the incident wave $a_i$ and outgoing wave $b_i$ for subsystem $i$:
\begin{equation}
a_i := \frac{e_i + \gamma f_i}{\sqrt{2\gamma}}, \qquad
b_i := \frac{e_i - \gamma f_i}{\sqrt{2\gamma}}.
\label{eq:scattering_transform}
\end{equation}

The inverse transformation is
\begin{equation}
e_i = \sqrt{\frac{\gamma}{2}}\,(a_i + b_i), \qquad
f_i = \frac{1}{\sqrt{2\gamma}}\,(a_i - b_i).
\label{eq:inverse_scattering}
\end{equation}

In the wave domain, port power has the norm expression:
\begin{equation}
e_i^\top f_i = \frac{1}{2}\big(\|a_i\|^2 - \|b_i\|^2\big).
\label{eq:power_norm}
\end{equation}

Hence, discrete passivity is equivalent in the wave domain to a norm condition:
\begin{equation}
\|b_i\| \leq \|a_i\|.
\label{eq:wave_passivity}
\end{equation}

In the terminology of monotone operators, the discrete passivity in~\eqref{eq:wave_passivity} is nonexpansiveness (NE). The subsystem does not amplify wave signals, corresponding to $\ell_2$ gain $\leq 1$.

For this linear lossless coupling,
\begin{equation}
e_A = e_B, \qquad f_A + f_B = 0.
\label{eq:lossless_coupling}
\end{equation}

This implies $e_A^\top f_A + e_B^\top f_B = 0$. Substituting~\eqref{eq:inverse_scattering} into~\eqref{eq:lossless_coupling}, the coupling constraints in the wave domain take the form
\begin{equation}
a_A = b_B, \qquad a_B = b_A.
\label{eq:wave_coupling}
\end{equation}

Define the stacked incident wave $a := \operatorname{col}(a_A, a_B) \in \mathbb{R}^{2m}$ and stacked outgoing wave $b := \operatorname{col}(b_A, b_B) \in \mathbb{R}^{2m}$; the coupling~\eqref{eq:wave_coupling} can then be written as
\begin{equation}
a = P b, \qquad P := \begin{bmatrix} 0 & I_m \\ I_m & 0 \end{bmatrix},
\label{eq:coupling_matrix}
\end{equation}
where $P^\top P = I_{2m}$ is an orthogonal matrix. For the more general case of $N$ subsystems with linear lossless interconnection, the coupling can be expressed as $a = P b$ with $P$ a general orthogonal matrix. Without loss of generality, this paper works throughout with the two-port swap matrix $P$.

From the orthogonality of the wave-domain coupling we have:
\begin{equation}
\|a\|^2 = \|P b\|^2 = b^\top P^\top P b = \|b\|^2.
\label{eq:coupling_norm}
\end{equation}

In the original $(e_i, f_i)$ coordinates the lossless constraint is an algebraic condition, whereas in the wave domain it becomes a norm equality. This transformation allows convex analysis and monotone operator tools to be applied to passivity.

\begin{remark}[Linear lossless coupling]
\label{rem:linear_coupling_scope}
Any network of ideal transformers, gyrators, and port permutations between collocated power ports (force--velocity, voltage--current, pressure--flow) admits a wave-domain representation $a = Pb$ with constant orthogonal $P$. State-modulated coupling, where $P$ depends on the subsystem states (configuration-dependent transmission ratios, variable contact topology), is not treated here. When the modulating state is frozen per macro-step and the resulting coupling matrix remains orthogonal, the same analysis applies locally to that frozen interconnection. Genuinely nonlinear lossless interconnections lie beyond the present scope.
\end{remark}

\subsection{Frozen Port Map and Macro-Step Consistency Condition}
\label{subsec:frozen_port_map}

This paper concerns discrete-time simulation. Let $\Delta t$ be the macro time step and $n = 0, 1, 2, \dots$ the macro-step index. All discrete wave variables absorb $\sqrt{\Delta t}$ relative to the continuous-time definitions~\eqref{eq:scattering_transform}, so that $a_i^n := \sqrt{\Delta t}\, a_i(t_n)$, $b_i^n := \sqrt{\Delta t}\, b_i(t_n)$, and $\frac{1}{2}(\|a_i^n\|^2 - \|b_i^{n+1}\|^2)$ is an energy. Every wave variable with a discrete superscript in the paper is understood to include this factor.

At macro-step $n$, each subsystem is driven by a discrete-passive integration operator $\Phi^{\Delta t}_i$. Given the current states $x_A^n, x_B^n$ of the two subsystems, the goal of each macro-step is to find the next states and port variables $(x_A^{n+1}, x_B^{n+1}, a_A, a_B, b_A, b_B)$ satisfying the three groups of equations:
\begin{equation}
\begin{cases}
(x_A^{n+1},b_A) = \Phi^{\Delta t}_A(x_A^n,a_A),\\[2pt]
(x_B^{n+1},b_B) = \Phi^{\Delta t}_B(x_B^n,a_B),\\[2pt]
a_A = b_B, \ a_B = b_A,
\end{cases}
\label{eq:monolithic_macro_step}
\end{equation}
which are, respectively, the internal dynamics of subsystems $A$ and $B$, and the strict coupling relation. We call this the \emph{monolithic macro-step}. Each macro-step requires solving an implicit algebraic system of dimension $(n_A+n_B+4m)$, which cannot be parallelized and is slow.

To decouple the interface from the state update, we freeze the current state and retain only the port-level input--output behavior. For subsystem $i$, fixing $x_i^n$ yields the \emph{frozen port map} of the current macro-step:
\begin{equation}
b_i = S_i^n(a_i), \quad i \in \{A, B\}.
\label{eq:frozen_port_map}
\end{equation}

$S_i^n(\cdot)$ is the outgoing wave produced by one evaluation of $\Phi_i^{\Delta t}(x_i^n, \cdot)$, with the resulting state discarded. The frozen port maps of the two subsystems stack into a block-diagonal operator:
\begin{equation}
S^n := \operatorname{diag}(S_A^n, S_B^n): \mathbb{R}^{2m} \to \mathbb{R}^{2m}.
\label{eq:stacked_frozen_map}
\end{equation}

The \emph{monolithic interface condition} of the current macro-step can then be written as
\begin{equation}
b = S^n(a), \quad a = P b.
\label{eq:monolithic_interface}
\end{equation}
Eliminating $a$ (or $b$) yields the equivalent fixed-point equation
\begin{equation}
b = S^n(P b) \quad \Longleftrightarrow \quad a = P S^n(a).
\label{eq:fixed_point}
\end{equation}

We call a solution of~\eqref{eq:monolithic_interface} the \emph{monolithic solution} at macro-step $n$, denoted $(a^{n,\star}, b^{n,\star})$. In fully coupled simulation, each macro-step solves a global implicit system to obtain it; our objective is to approximate this solution in an energy-safe manner. The fixed-point problem~\eqref{eq:fixed_point} is the central object.

\section{Proposed Method: Wave-Domain DRS Interface Algorithm}
\label{sec:iterative_coupling}

\begin{figure}[!t]
\centering
\includegraphics[width=\linewidth]{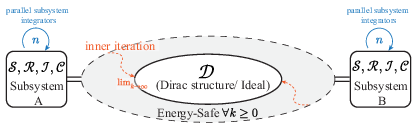}%
\caption{Parallel iterative coupling with macro-step $n$ and inner iteration $k$.}
\label{fig:iterative_coupling}
\end{figure}

\subsection{From Monolithic Fixed-Point to Interface Coordination}
\label{subsec:monolithic_to_coordination}

Section~\ref{subsec:frozen_port_map} formulated the monolithic interface condition as the fixed-point equation $b = S^n(Pb)$. We now turn this fixed-point problem into an interface coordination scheme. The equation expresses two competing consistency requirements on the outgoing wave $b$. The subsystems demand that $b$ be the reflection of their incident wave via the frozen port map $S^n$, while the coupling demands that $b$ satisfy the lossless interconnection $a = Pb$ through the orthogonal matrix $P$.

The most direct approach is to apply Picard iteration $b^{k+1} = S^n(P b^k)$ to this fixed-point problem. However, it lacks the structural property that the algorithm's energy does not increase, and it provides no way to judge whether early termination has injected spurious energy. We therefore adopt Douglas-Rachford splitting (DRS). DRS alternates between the \emph{subsystem response} side and the \emph{coupling correction} side: each queries its operator, and an auxiliary variable is updated toward consensus. Each step can be interpreted as a passive operation in the wave domain.

\subsection{Two Operators for Splitting: Subsystem Response and Coupling Correction}
\label{subsec:two_operators}

Without loss of generality, consider DRS for two maps $T_1, T_2$ \cite{bauschke2017convex}.
In our coupling problem, the space is $\mathbb{R}^{2m}$, the dimension of the stacked wave variables, and the two maps are:
\begin{equation}
T_1 = S^n = \begin{bmatrix} S_A^n & 0 \\ 0 & S_B^n \end{bmatrix}, \qquad
T_2 = J_L := (2I - P)^{-1}.
\label{eq:drs_mappings}
\end{equation}
 
\textbf{$T_1 = S^n$ is the frozen port map of the subsystems.} The block-diagonal structure of $S^n$ means that $S_A^n$ and $S_B^n$ can be evaluated \textbf{fully in parallel} with no communication. This is the root of the method's parallelism.

\textbf{$T_2 = J_L$ is the coupling resolvent operator.} Define the coupling mismatch operator $L := I - P$. Then $J_L = (I+L)^{-1} = (2I-P)^{-1}$ is a closed-form $2m \times 2m$ matrix operation that does not involve the internal states of the subsystems. For the two-port swap matrix $P = \begin{bmatrix} 0 & I \\ I & 0 \end{bmatrix}$, we have $J_L(x) = \frac{1}{3}(2x + Px)$. For general orthogonal $P$, $J_L = (2I-P)^{-1}$ is well defined.

With parallel execution of $S_A^n$ and $S_B^n$, the wall-clock time per inner iteration is the maximum of the two frozen port query times plus a small coupling update. The coupling update is a closed-form operation on a $2m \times 2m$ matrix. For typical co-simulation settings $m$ is small, so the cost is dominated by the slower subsystem query.

\subsection{DRS Inner Iteration Algorithm}
\label{subsec:drs_inner_iteration}

Taking $T_1 = S^n$ and $T_2 = J_L$, we obtain the DRS inner iteration for macro-step $n$:
\begin{equation}
\begin{aligned}
\text{(Response step)}\quad &\hat{b}^{n,k} := S^n(u^{n,k}), \\[4pt]
\text{(Projection step)}\quad &w^{n,k} := J_L\big(2\hat{b}^{n,k} - u^{n,k}\big), \\[4pt]
\text{(Update step)}\quad &u^{n,k+1} := u^{n,k} + w^{n,k} - \hat{b}^{n,k}.
\end{aligned}
\label{eq:drs_inner_iteration}
\end{equation}

Here $u^{n,k}$ is the current guess of the coupling-consistent incident wave, $\hat{b}^{n,k}$ the corresponding subsystem response, and $w^{n,k}$ the coupling projection. 
When $\hat{b}^{n,k} = w^{n,k}$, the response $\hat{b}^{n,k}$ already satisfies the coupling, the update is zero, and a fixed point is attained.

Section~\ref{sec:passivity_convergence} will prove that when $S^n$ and $J_L$ satisfy the FNE conditions, the DRS single-step operator is also FNE, and consequently the inner iteration satisfies Fejér monotonicity. The algorithm thus possesses a natural Lyapunov function, dissipating the distance to a fixed point at every step.

At a DRS fixed point, $u^\dagger = u^\dagger + J_L(2b^\star - u^\dagger) - b^\star$ with $b^\star := S^n(u^\dagger)$. Hence $J_L(2b^\star - u^\dagger) = b^\star$, and substituting $J_L = (2I-P)^{-1}$ gives
\begin{equation}
2b^\star - u^\dagger = (2I-P)b^\star \;\Longrightarrow\; u^\dagger = Pb^\star.
\label{eq:fixed_point_verification}
\end{equation}
Thus $b^\star = S^n(Pb^\star)$, the monolithic condition~\eqref{eq:fixed_point}.

\subsection{Finite-Iteration Termination and Macro-Step Advancement}
\label{subsec:finite_iteration}

In practice, only a finite number of inner iterations are executed per macro-step. Let $K_n \geq 0$ be the budget for macro-step $n$. The initialization is:
\begin{equation}
b^n := \operatorname{col}(b_A^n, b_B^n), \qquad u^{n,0} = P b^n, \qquad \hat{b}^{n,\star} = b^n.
\label{eq:initialization}
\end{equation}
If $K_n = 0$, the inner iteration is skipped and $\hat{b}^{n,\star} = b^n$; the algorithm reduces to an explicit wave interface $\hat{a}^n = P b^n$.

If $K_n > 0$, we execute the DRS inner iteration (Eq.~\eqref{eq:drs_inner_iteration}) for $K_n$ steps, taking the $\hat{b}^{n,k}$ computed in the last iteration:
\begin{equation}
\hat{b}^{n,\star} := \hat{b}^{n,K_n-1}.
\label{eq:final_response}
\end{equation}
The convergence criterion $\|u^{n,k+1} - u^{n,k}\| \leq \varepsilon$ can trigger early exit.

We then construct the coupling-consistent incident wave and advance the subsystems in parallel:
\begin{equation}
\hat{a}^n := P \hat{b}^{n,\star}, \quad
(x_i^{n+1}, b_i^{n+1}) = \Phi_i^{\Delta t}(x_i^n, \hat{a}_i^n), \quad i \in \{A, B\}.
\label{eq:macro_step_advance}
\end{equation}

\begin{algorithm}[t]
\caption{Wave-Domain Douglas-Rachford Iterative Coupling}
\label{alg:scattering_iterative_coupling}
\begin{algorithmic}[1]
\Require Initial states $x_A^0, x_B^0$; initial outgoing waves $b_A^0, b_B^0$; macro time step $\Delta t$; coupling matrix $P$ ($P^\top P = I$); inner iteration budget $K_n$ ($n = 0,1,\dots$); optional convergence tolerance $\varepsilon > 0$
\Ensure Updated states $x_i^{n+1}$ and outgoing waves $b_i^{n+1}$ for each macro-step

\For{$n = 0, 1, 2, \dots$}
    \State // Freeze state and form frozen port map
    \State Form $S^n = \operatorname{diag}(S_A^n, S_B^n)$ from $x_A^n, x_B^n$ and $\Delta t$
    \State // Initialize inner iteration
    \State $b^n \gets \operatorname{col}(b_A^n, b_B^n)$
    \State $u^{n,0} \gets P b^n$
    \State $\hat{b}^{n,\star} \gets b^n$
    \State // Inner iteration
    \For{$k = 0, 1, \dots, K_n-1$}
        \State $\hat{b}^{n,k} \gets S^n(u^{n,k})$ \Comment{Evaluate $S_A^n, S_B^n$ in parallel}
        \State $w^{n,k} \gets J_L(2\hat{b}^{n,k} - u^{n,k})$ \Comment{$J_L = (2I-P)^{-1}$, closed form}
        \State $u^{n,k+1} \gets u^{n,k} + w^{n,k} - \hat{b}^{n,k}$
        \State $\hat{b}^{n,\star} \gets \hat{b}^{n,k}$
        \If{$\|u^{n,k+1} - u^{n,k}\| \leq \varepsilon$}
            \State \textbf{break}
        \EndIf
    \EndFor
    \State // Macro-step
    \State $\hat{a}^n \gets P \hat{b}^{n,\star}$ \Comment{Coupling-consistent incident wave}
    \For{$i \in \{A, B\}$ in parallel}
        \State $(x_i^{n+1}, b_i^{n+1}) \gets \Phi_i^{\Delta t}(x_i^n, \hat{a}_i^n)$\Comment{one step integrator}
    \EndFor
\EndFor
\end{algorithmic}
\end{algorithm}

\section{Finite-Iteration Energy Safety and Convergence Guarantees}
\label{sec:passivity_convergence}

\subsection{Standing Conditions}
\label{subsec:standing_conditions}

The DRS algorithm of Section~III requires $S^n$ and $J_L$ to be firmly nonexpansive (FNE). The FNE property of $J_L$ follows automatically from the coupling structure (\S~\ref{subsec:fejer_monotonicity}). For $S^n$, we impose the following standing condition. The physical energy of the macro-step advancement is guaranteed by Condition~\ref{ass:discrete_passivity}.

\begin{condition}[Firm nonexpansiveness of frozen port map]
\label{ass:fne_portmap}
At every macro-step $n$, the frozen port map $S_i^n$ of each subsystem is firmly nonexpansive (FNE):
\begin{equation}
\|S_i^n(\alpha) - S_i^n(\beta)\|^2 \leq \big\langle S_i^n(\alpha) - S_i^n(\beta),\; \alpha - \beta \big\rangle, \quad \forall \alpha, \beta \in \mathbb{R}^m.
\label{eq:fne_inequality}
\end{equation}
Since $S^n = \operatorname{diag}(S_A^n, S_B^n)$, if each $S_i^n$ is FNE then the stacked map $S^n$ is also FNE.
\end{condition}

FNE is stronger than nonexpansiveness ($\|b\| \leq \|a\|$ in~\eqref{eq:wave_passivity}). For a linear scalar system $b = \rho \cdot a$, NE requires $\rho \in [-1, 1]$ (no signal amplification); FNE additionally requires $\rho \geq 0$ (no phase inversion).

Treating $S_i^n$ as a static map with input $a$, output $b$ in wave-domain I/O coordinates, the FNE inequality rearranges to
\begin{equation}
\|b_1 - b_2\|^2 \leq \langle b_1 - b_2, a_1 - a_2\rangle,
\label{eq:fne_rearranged}
\end{equation}
a static system has zero stored energy, so the inequality above states that $S_i^n$ is incrementally \textbf{output-strictly passive} (OSP). Condition~\ref{ass:fne_portmap} therefore requires $S_i^n$ to be OSP, a condition stricter than passivity.

By Minty's theorem \cite{bauschke2017convex}, a map is FNE if and only if it is the resolvent $S = (I + A)^{-1}$ of a maximal monotone operator $A$. Resolvents include the proximal operator and implicit Euler steps.

For a local analysis, linearize the frozen port map at the current operating point: $\delta b = J_S \,\delta a$, where $J_S$ is the Jacobian of $S_i^n$. Substituting into~\eqref{eq:fne_inequality} gives the matrix FNE condition
\begin{equation}
J_S^\top J_S \preceq \frac{J_S + J_S^\top}{2}.
\label{eq:general_fne}
\end{equation}
This inequality can be checked numerically from the subsystem dynamics, as demonstrated in Section~V. The following remark describes a common case in which the condition simplifies to a direct choice of $\gamma$.

\begin{remark}[$\gamma$-tuning when the port impedance is symmetric] 
\label{rem:gamma_rule}
Condition~\ref{ass:fne_portmap} simplifies when the incremental port response is reciprocal, meaning that the port variables $(e,f)$ satisfy $\delta e = Z_d \,\delta f$ with a symmetric positive definite impedance $Z_d = Z_d^\top \succ 0$. This condition is satisfied by standard mechanical ports (springs, dampers, inertias) and their electrical and hydraulic analogues. Gyroscopic terms violate reciprocity. Substituting the inverse scattering transform~\eqref{eq:inverse_scattering} gives
\begin{equation}
J_S = (Z_d - \gamma I)(Z_d + \gamma I)^{-1}.
\label{eq:js_expression}
\end{equation}
Since $Z_d$ is symmetric, $J_S$ is symmetric, and~\eqref{eq:general_fne} simplifies to $0 \preceq J_S \preceq I$. The eigenvalues of $Z_d$ are $\lambda_j > 0$, so the eigenvalues of $J_S$ are $(\lambda_j - \gamma)/(\lambda_j + \gamma)$. The condition $0 \preceq J_S \preceq I$ is therefore equivalent to $\lambda_j \geq \gamma$ for all $j$, which is
\begin{equation}
Z_d \succeq \gamma I.
\label{eq:impedance_condition}
\end{equation}
Choosing $\gamma \leq \lambda_{\min}(Z_d)$ guarantees Condition~\ref{ass:fne_portmap} locally. The impedance $Z_d$ can be obtained from a linearization of the subsystem dynamics at the current operating point. For non-reciprocal ports (e.g., systems with gyroscopic terms),~\eqref{eq:general_fne} should be checked directly.
\end{remark}

\begin{condition}[Discrete passivity of subsystem integrator]
\label{ass:discrete_passivity}
For each subsystem $i \in \{A, B\}$, its discrete-time advancement map $\Phi_i^{\Delta t}$ satisfies, for all macro-steps $n$:
\begin{equation}
H_i(x_i^{n+1}) - H_i(x_i^n) \leq \frac{1}{2}\big(\|a_i^n\|^2 - \|b_i^{n+1}\|^2\big),
\label{eq:discrete_passivity_scattering}
\end{equation}
where $a_i^n$ is the incident wave used to advance subsystem $i$, and $b_i^{n+1}$ is the outgoing wave produced after the advancement.
\end{condition}

This is the discrete-time counterpart of the continuous passivity $\dot{H}_i \leq e_i^\top f_i$ (Eq.~\eqref{eq:continuous_passivity}). The change in stored physical energy is bounded above by the net wave energy supplied through the port.

\begin{remark}[Integrators that satisfy Condition~\ref{ass:discrete_passivity}]
Discrete-gradient methods~\cite{robert2014Discrete,celledoni2012preserving} and algebraically stable implicit Runge--Kutta methods (Gauss, Radau~IIA, Lobatto~IIIC)~\cite{wanner1996solving} satisfy Condition~\ref{ass:discrete_passivity} by construction. For quadratic Hamiltonians these reduce to the implicit midpoint rule.
\end{remark}

\subsection{Fejér Monotonicity of DRS}
\label{subsec:fejer_monotonicity}

Condition~\ref{ass:fne_portmap} ensures $S^n$ is FNE. We now show that $J_L$ inherits the same property from the coupling structure. Since $P^\top P = I$, $L := I-P$ satisfies $x^\top L x = \|x\|^2 - x^\top P x \ge 0$, hence $L$ is linear maximal monotone. By Minty's theorem, $J_L = (I+L)^{-1}$ is FNE. For $P=\begin{bmatrix}0&I\\I&0\end{bmatrix}$ the closed form is $J_L(x)=\frac13(2x+Px)$.

Both $S^n$ and $J_L$ are therefore FNE, hence their reflectors $R_S^n := 2S^n - I$ and $R_L := 2J_L - I$ are nonexpansive (NE). The DRS single-step operator $T_{\text{DR}}^n := \frac{1}{2}(I + R_L R_S^n)$ is therefore FNE~\cite{bauschke2017convex}, and the iterates satisfy Fejér monotonicity: for any fixed point $u^{n,\dagger}$,
\begin{equation} 
\|u^{n,k+1} - u^{n,\dagger}\|^2 \leq \|u^{n,k} - u^{n,\dagger}\|^2 - \|u^{n,k+1} - u^{n,k}\|^2.
\label{eq:fejer}
\end{equation}

Hence, the DRS inner iteration dissipates its algorithmic storage relative to a fixed point at every step.

\subsection{Theorem 1: Finite-Iteration Augmented Storage Inequality}
\label{subsec:theorem1}
Theorem~1 follows from adding the physical energy inequality (Condition~\ref{ass:discrete_passivity}) and the algorithmic Lyapunov inequality (Fejér monotonicity).

Consider macro-step $n$. On the physical side, Condition~\ref{ass:discrete_passivity} applied to each subsystem and summed gives:
\begin{equation}
\begin{aligned}
\big(H_A(x_A^{n+1}) + H_B(x_B^{n+1})\big) - \big(H_A(x_A^n) + H_B(x_B^n)\big)\\
\leq \frac{1}{2}\big(\|\hat{a}^n\|^2 - \|b^{n+1}\|^2\big),
\label{eq:physical_sum}
\end{aligned}
\end{equation}
where $\hat{a}^n = P \hat{b}^{n,\star}$ is the incident wave used to drive the macro-step after inner iteration terminates, and $b^{n+1} = \operatorname{col}(b_A^{n+1}, b_B^{n+1})$ is the outgoing wave produced after advancement.

On the algorithmic side, Fejér monotonicity from \S~\ref{subsec:fejer_monotonicity} summed over $K_n$ steps yields:
\begin{equation}
\frac{1}{2}\|u^{n,K_n} - u^{n,\dagger}\|^2 - \frac{1}{2}\|u^{n,0} - u^{n,\dagger}\|^2 \leq 0.
\label{eq:fejer_sum}
\end{equation}

Adding~\eqref{eq:physical_sum} and~\eqref{eq:fejer_sum} and defining
\begin{equation}
\mathcal{V}^n := H_A(x_A^n) + H_B(x_B^n) + \frac{1}{2}\|u^{n,K_n} - u^{n,\dagger}\|^2,
\label{eq:augmented_storage_def}
\end{equation}
we obtain the following theorem. The quantity $\mathcal{V}^n$ is defined per macro-step $n$ and does not chain across steps.

\begin{theorem}[Finite-iteration Augmented Storage Inequality]
\label{thm:augmented_storage}
Assume Condition~\ref{ass:fne_portmap} and Condition~\ref{ass:discrete_passivity} hold, and that the monolithic solution $b^{n,\star}$ (defined in \S~\ref{subsec:frozen_port_map}) exists at macro-step $n$. Let the inner iteration of macro-step $n$ be executed for $K_n \geq 0$ steps according to Algorithm~\ref{alg:scattering_iterative_coupling}, and let the macro-step be advanced with $\hat{b}^{n,\star}$ and incident wave $\hat{a}^n = P \hat{b}^{n,\star}$. Then
\begin{equation}
\begin{aligned}
&\big(H_A(x_A^{n+1}) + H_B(x_B^{n+1})\big) - \big(H_A(x_A^n) + H_B(x_B^n)\big) \\
&\quad + \frac{1}{2}\|u^{n,K_n} - u^{n,\dagger}\|^2 - \frac{1}{2}\|u^{n,0} - u^{n,\dagger}\|^2 \\
&\qquad \leq \frac{1}{2}\big(\|\hat{a}^n\|^2 - \|b^{n+1}\|^2\big).
\end{aligned}
\label{eq:augmented_storage_inequality}
\end{equation}
\end{theorem}

\begin{proof}
Summing Condition~\ref{ass:discrete_passivity} over the subsystems yields~\eqref{eq:physical_sum}. Fejér monotonicity from \S~\ref{subsec:fejer_monotonicity} yields~\eqref{eq:fejer_sum}. Adding the two inequalities gives~\eqref{eq:augmented_storage_inequality}.
\end{proof}

Inequality~\eqref{eq:augmented_storage_inequality} states that the change in physical energy $\Delta H = (H_A+H_B)(x^{n+1}) - (H_A+H_B)(x^n)$, reduced by the Fejér decrease of the algorithmic storage, is bounded above by the net wave energy $\frac{1}{2}(\|\hat{a}^n\|^2 - \|b^{n+1}\|^2)$ supplied through the interface. When the inner iteration has not fully converged, the Fejér decrease is strictly positive and provides a dissipation margin that absorbs the interface mismatch.

The fixed point $u^{n,\dagger}$ is never computed online. The theorem guarantees that a nonnegative term $\frac{1}{2}\|u^{n,0}-u^{n,\dagger}\|^2 - \frac{1}{2}\|u^{n,K_n}-u^{n,\dagger}\|^2$ exists such that the energy inequality holds at every macro-step. The practical guarantee follows from Fejér monotonicity of the DRS iterates, which holds without explicit value of $u^{n,\dagger}$.

\subsection{Theorem 2: Hard-Coupling Limit Recovers Monolithic Discretization}
\label{subsec:theorem2}

Theorem~\ref{thm:augmented_storage} guarantees energy safety for any finite $K_n$. We now show that as $K_n \to \infty$, the partitioned update converges macro-step by macro-step to a monolithic discrete-time trajectory.

By Condition~\ref{ass:fne_portmap} and Minty's theorem, there exists a maximal monotone operator $A^n$ such that $S^n = J_{A^n} = (I + A^n)^{-1}$. $L = I-P$ is linear maximal monotone (\S~\ref{subsec:fejer_monotonicity}). Hence the DRS iteration solves $0 \in A^n(b) + (I-P)b$, the zero of the sum of two monotone operators. The operator $A^n$ is used only for analysis and is never constructed in the actual algorithm.

\begin{lemma}[DRS inner iteration converges to monolithic wave solution]
\label{lem:drs_convergence}
Assume Condition~\ref{ass:fne_portmap} holds and the monolithic solution $b^{n,\star}$ exists at macro-step $n$. Then as $k \to \infty$, the DRS auxiliary sequence $u^{n,k}$ converges to some fixed point $u^{n,\dagger} \in \operatorname{Fix}(T_{\text{DR}}^n)$, and the sequence $\hat{b}^{n,k} = J_{A^n}(u^{n,k})$ converges to the monolithic outgoing wave $b^{n,\star}$.
\end{lemma}

\begin{proof}
In finite-dimensional spaces, DRS iteration for the sum of maximal monotone operators converges strongly \cite{lions1979splitting,bauschke2017convex}. Condition~\ref{ass:fne_portmap} and Minty's theorem give $S^n = J_{A^n}$ with $A^n$ maximal monotone. $L = I-P$ is linear maximal monotone (\S~\ref{subsec:fejer_monotonicity}). The standard DRS convergence theorem directly yields $u^{n,k} \to u^{n,\dagger} \in \operatorname{Fix}(T_{\text{DR}}^n)$. By Lipschitz continuity of the resolvent (implied by FNE), $\hat{b}^{n,k} = J_{A^n}(u^{n,k}) \to J_{A^n}(u^{n,\dagger}) = b^{n,\star}$.
\end{proof}

\begin{theorem}[Hard-coupling limit recovers monolithic discretization]
\label{thm:kn_to_infty_monolithic}
Assume Condition~\ref{ass:fne_portmap} holds at every macro-step, the monolithic solution exists, and the subsystem advancement maps $\Phi_i^{\Delta t}$ and frozen port maps $S_i^n$ are continuous in state and incident wave over compact trajectory sets. Then for any finite time interval $T = N \Delta t$, if $K_n \to \infty$ for all $n = 0, \dots, N-1$, the partitioned trajectory converges pointwise on the grid to an admissible monolithic discrete trajectory, one that satisfies the monolithic interface condition~\eqref{eq:monolithic_interface} at every macro-step.
\end{theorem}

\begin{proof}
By induction on $n$. For $n=0$ the initial states agree. Suppose the partitioned state $x^n$ has converged to the monolithic state. At macro-step $n$, Lemma~\ref{lem:drs_convergence} gives $\hat{b}^{n,\star} \to b^{n,\star}$, hence $\hat{a}^n = P \hat{b}^{n,\star} \to a^{n,\star}$. By continuity of $\Phi_i^{\Delta t}$ at $(x_i^n, \hat{a}_i^n)$,
\begin{equation}
(x_i^{n+1}, b_i^{n+1}) = \Phi_i^{\Delta t}(x_i^n, \hat{a}_i^n) \to \Phi_i^{\Delta t}(x_i^n, a_i^{n,\star}),
\end{equation}
and $\Phi_i^{\Delta t}(x_i^n, a_i^{n,\star})$ is the monolithic update at macro-step $n$, since $(a^{n,\star}, b^{n,\star})$ satisfies the monolithic condition~\eqref{eq:monolithic_interface}. The induction step is complete.
\end{proof}

\begin{remark}[Uniqueness of the monolithic solution]
\label{rem:uniqueness}
Condition~\ref{ass:fne_portmap} and Minty's theorem give $S^n = (I+A^n)^{-1}$ with $A^n$ maximal monotone. The monolithic interface condition is equivalent to $0 \in A^n(b) + (I-P)b$. When $A^n$ is strictly monotone, $A^n + (I-P)$ is strictly monotone and the monolithic solution at macro-step $n$ is unique. Strict monotonicity holds, for example, for positive stiffness springs, positive damping, and monotone nonlinearities.

When $A^n$ is monotone but not strictly monotone, multiple solutions may coexist. Physical sources include Coulomb friction, ideal plasticity, and statically indeterminate contact. Fej\'er monotonicity implies that DRS selects the fixed point closest to the initialization $u^{n,0}=P b^n$, a wave-domain minimal-jump selection. The same non-uniqueness is present in the monolithic problem. The coupling method inherits it.
\end{remark}

% !TEX root = main.tex
\section{Experiments: Two-Oscillator Benchmark}
\label{sec:experiments}

We validate the iterative coupling framework on a minimal nontrivial benchmark, two coupled oscillators split into two subsystems with one wave port each. Its difficulty comes from a \emph{mixed Duffing--linear} configuration. Subsystem~A includes a stiff hardening spring nonlinearity, so the frozen discrete-time port maps are state-dependent and cannot be reduced to a globally linear impedance bound. The hardening spring $k_{\mathrm{nl}} q_1^3$ has derivative $3k_{\mathrm{nl}} q_1^2 \ge 0$ and $k_1 > 0$, so the subsystem constitutive relation is strictly monotone. By Remark~4 the monolithic solution at each macro-step is unique for this benchmark. The experiments target three claims from Sections~II--IV: (i) Condition~\ref{ass:fne_portmap} (FNE margins), (ii) the augmented-storage energy guarantee under finite inner iterations (Theorem~\ref{thm:augmented_storage}), and (iii) convergence to a monolithic reference trajectory as $K_n\to\infty$ (Theorem~\ref{thm:kn_to_infty_monolithic}).
\begin{figure}[h]
  \centering
  \includegraphics[width=\linewidth]{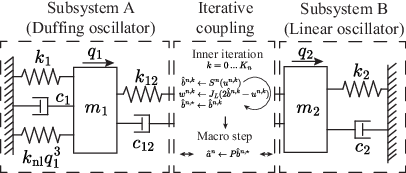}%
  \caption{Schematic of the two-oscillator benchmark.}
  \label{fig:twoosc_schematic}
\end{figure}
\subsection{Benchmark and settings.}
Fig.~\ref{fig:twoosc_schematic} depicts the subsystem decomposition and ports. Each subsystem exposes a single port $(e_i,f_i)$ and is interfaced through wave variables $(a_i,b_i)$ with lossless coupling $a=Pb$ (Section~\ref{subsec:scattering}). Both subsystems are time-discretized by a passivity-preserving discrete-gradient method, so Condition~\ref{ass:discrete_passivity} holds by construction. For the linear subsystem, this reduces to an implicit midpoint/trapezoidal update. All model and simulation parameters are listed in Table~\ref{tab:twoosc_params}.

\begin{table}[!t]
  \caption{Two-oscillator benchmark parameters and simulation settings.}
  \label{tab:twoosc_params}
  \centering
  \renewcommand{\arraystretch}{1.1}
  \begin{tabular}{l c l c}
    \hline
    Parameter & Value & Parameter & Value \\
    \hline
    $m_1$ & $8.0\;\mathrm{kg}$ & $m_2$ & $4.0\;\mathrm{kg}$ \\
    $k_1$ & $100\;\mathrm{N\,m^{-1}}$ & $k_2$ & $50\;\mathrm{N\,m^{-1}}$ \\
    $c_1$ & $0$ & $c_2$ & $0$ \\
    $k_{12}$ & $120\;\mathrm{N\,m^{-1}}$ & $c_{12}$ & $0.05\;\mathrm{N\,s\,m^{-1}}$ \\
    $k_{\mathrm{nl}}$ & $8000\;\mathrm{N\,m^{-3}}$ & $q_1(0)$ & $0.4\;\mathrm{m}$ \\
    $q_2(0)$ & $0$ & $v_1(0)=v_2(0)$ & $0$ \\
    $\Delta t$ & $0.01\;\mathrm{s}$ & $T$ & $10.0\;\mathrm{s}$ \\
    $\gamma$ & $0.4\;\mathrm{N\,s\,m^{-1}}$ & $K_n$ & $\{0,3,8,20,35,50\}$ \\
    \hline
  \end{tabular}
\end{table}

The scattering impedance is set to $\gamma = 0.4\ \mathrm{N\,s\,m^{-1}}$. Remark~\ref{rem:gamma_rule} states that Condition~\ref{ass:fne_portmap} holds locally when the incremental port impedance satisfies $Z_d \succeq \gamma I$. The linear subsystem yields $\lambda_{\min}(Z_d) \approx 0.5\ \mathrm{N\,s\,m^{-1}}$ at $\Delta t = 0.01\ \mathrm{s}$. The Duffing subsystem has state-dependent impedance. The hardening spring $k_{\mathrm{nl}}>0$ keeps its contribution positive. Choosing $\gamma = 0.4$ below the linear bound accommodates this variation. Condition~\ref{ass:fne_portmap} is numerically supported along the realized trajectories by the sampled FNE margins in Fig.~\ref{fig:twoosc_fne_passivity}(a).

The open-source C++ implementation of this benchmark, including pre-computed results and scripts to reproduce the paper figures, is available at \url{https://github.com/V7CN/wave-domain-iterative-coupling}.

\subsection{Results and analysis.} 
Fig.~\ref{fig:twoosc_traj} compares the trajectories $q_1(t)$ and $q_2(t)$ for different inner-iteration budgets $K_n$ against a monolithic reference computed with the same passivity-preserving discrete-gradient time stepper and time step $\Delta t$.

\begin{figure}[ht] 
  \centering
  \includegraphics[width=\linewidth]{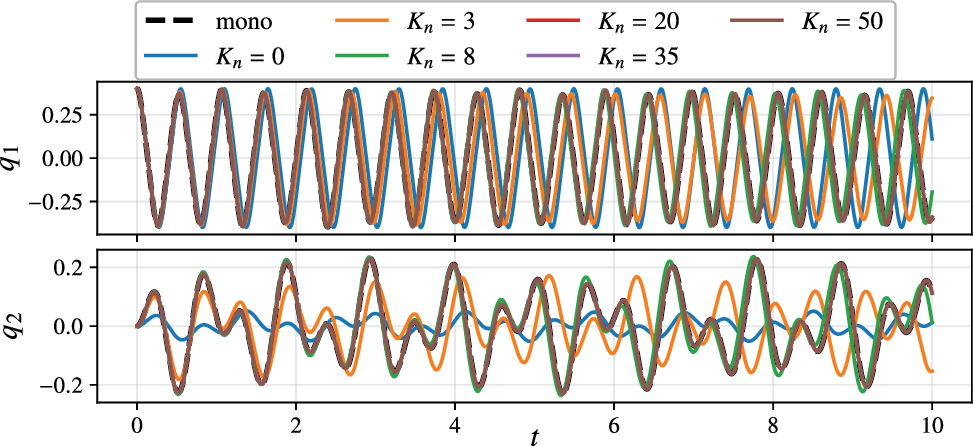}%
  \caption{Trajectories $(q_1,q_2)$ for different inner-iteration budgets $K_n$ and the monolithic reference.}
  \label{fig:twoosc_traj}
\end{figure}

We check Condition~\ref{ass:fne_portmap} at each macro-step $n$ by estimating the margin $\Delta_i^n(\alpha,\beta) := \langle S_i^n(\alpha)-S_i^n(\beta),\,\alpha-\beta\rangle - \|S_i^n(\alpha)-S_i^n(\beta)\|^2$ on 32 randomly sampled incident wave pairs $(\alpha,\beta)$. The inequality holds at the tested points when $\Delta_i^n(\alpha,\beta)\ge 0$. 
Fig.~\ref{fig:twoosc_fne_passivity}(a) reports the worst sampled FNE margin $\min_{n,\alpha,\beta}\Delta_i^n(\alpha,\beta)$ for each subsystem and each inner-iteration budget. The zero line marks the FNE threshold. All reported margins remain nonnegative up to the numerical tolerance used in the random pair tests.

For discrete passivity, we monitor the residual $r_i^n := (H_i(x_i^{n+1})-H_i(x_i^n))-\tfrac12(\|a_i^n\|^2-\|b_i^{n+1}\|^2)$; Condition~\ref{ass:discrete_passivity} holds at step $n$ whenever $r_i^n\le 0$. Fig.~\ref{fig:twoosc_fne_passivity}(b) reports the worst positive passivity residual $\max_{n,i} (r_i^n)_+$ and the worst positive augmented-storage residual across all macro-steps, for each inner-iteration budget $K_n$. To evaluate the augmented-storage residual offline, $u^{n,\dagger}$ is approximated by running DRS with a large diagnostic budget ($K_{\mathrm{diag}}=200$) at each macro-step.

Finally, Fig.~\ref{fig:twoosc_convergence}(a) shows the time history of the instantaneous state error for each $K_n$. Fig.~\ref{fig:twoosc_convergence}(b) shows that the maximum, time-RMS, and final-time state errors all decay monotonically over the tested budgets. This monotonic trend is consistent with the convergence limit stated in Theorem~\ref{thm:kn_to_infty_monolithic}.

\begin{figure}[!t]
  \centering
  \subfloat[]{\includegraphics[width=0.48\linewidth]{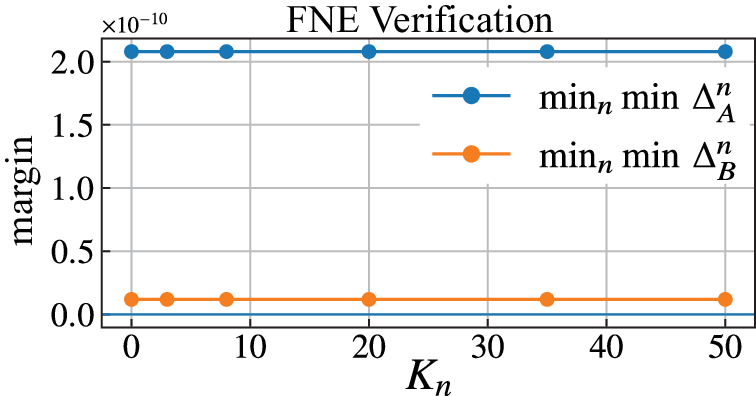}\label{fig:twoosc_fne}}%
  \hfill
  \subfloat[]{\includegraphics[width=0.48\linewidth]{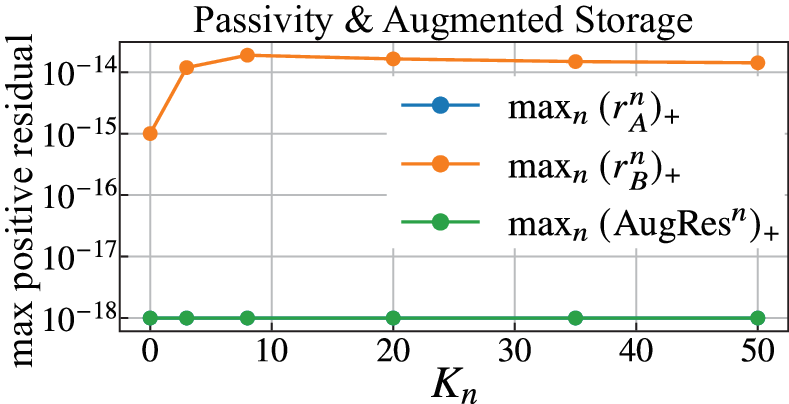}\label{fig:twoosc_passivity}}%
  \caption{Verification of the standing conditions and Theorem~1 versus $K_n$. (a) FNE verification (Condition~1): worst-case numerically evaluated margins. (b) Passivity and augmented storage (Condition~2 and Theorem~1): maximum positive parts over all macro-steps of the discrete-passivity residual \eqref{eq:discrete_passivity_scattering} and the augmented-storage residual \eqref{eq:augmented_storage_inequality}.}
  \label{fig:twoosc_fne_passivity}
\end{figure}

\begin{figure}[!t]
  \centering
  \subfloat[]{\includegraphics[width=0.48\linewidth]{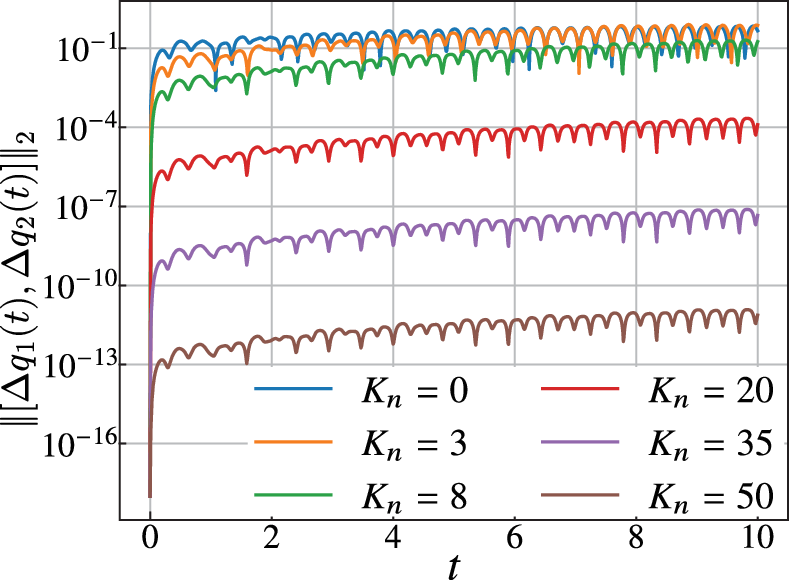}\label{fig:twoosc_poserr}}%
  \hfill
  \subfloat[]{\includegraphics[width=0.48\linewidth]{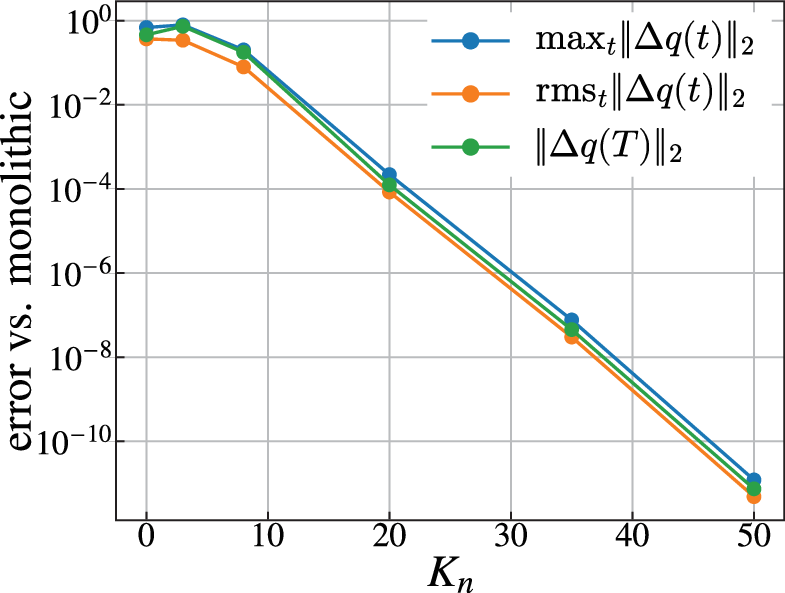}\label{fig:twoosc_knconv}}%
  \caption{Hard-coupling limit validation (Theorem~\ref{thm:kn_to_infty_monolithic}). (a) Time history of the instantaneous state error $\|[\Delta q_1(t),\Delta q_2(t)]\|_2$ with respect to the monolithic reference. (b) Maximum, time-RMS, and final-time error versus $K_n$ (all three decay monotonically over the tested budgets).}
  \label{fig:twoosc_convergence}
\end{figure}

In summary, for the selected $\gamma$, the sampled FNE margins along the realized trajectories are numerically nonnegative (Fig.~\ref{fig:twoosc_fne_passivity}(a)), indicating that Condition~\ref{ass:fne_portmap} is consistent with the benchmark. 
The discrete-passivity residuals and the augmented-storage residual are essentially nonpositive in the positive-part summaries (Fig.~\ref{fig:twoosc_fne_passivity}(b)), with any remaining positive part at numerical roundoff (on the order of $10^{-14}$ in double precision), supporting the finite-iteration guarantee of Theorem~\ref{thm:augmented_storage}. 
Finally, the maximum, time-RMS, and final-time state-error metrics decrease over the tested budgets and the trajectories approach the monolithic reference (Fig.~\ref{fig:twoosc_convergence}), consistent with Theorem~\ref{thm:kn_to_infty_monolithic}.

% !TEX root = main.tex
\section{Discussion}
\label{sec:discussion}

This letter presented an early-terminable energy-safe iterative coupling interface for partitioned port-Hamiltonian subsystems by embedding Douglas--Rachford splitting in the wave domain. The two standing conditions, firm nonexpansiveness of the frozen port maps and discrete passivity of the subsystem integrators, yield an augmented-storage passivity guarantee for any finite inner-iteration budget, and convergence to the monolithic discretization as the inner budget increases.

Several questions remain open. The present mathematical guarantees rely on the coupling being a linear orthogonal map $a=Pb$ and on the ability to freeze the subsystem state at the previous macro-step. Extending the framework to nonlinear lossless interconnections is an open direction. A quantitative state-error bound as a function of the finite inner-iteration budget $K_n$ is not yet available. When the monolithic interface condition admits multiple solutions, the DRS limit depends on initialization, and a principled mechanism for selecting among solutions remains an open question. 

Extensions of interest include adaptive selection of $\gamma$ and $K_n$ based on online FNE margin estimates, asynchronous inner iterations under communication delays, and integration of the interface layer into in-the-loop model-based control workflows.

\bibliographystyle{IEEEtran}
\bibliography{references} % Include the new bib file

\vspace{-1.5\baselineskip}

\end{document}